\documentclass[accepted]{uai2023} 

\usepackage[american]{babel}
\usepackage{amsfonts}
\usepackage{amssymb}
\usepackage{amsthm}

\usepackage{natbib} 
\usepackage{mathtools} 
\usepackage{booktabs} 
\usepackage{tikz} 
\usepackage{mathrsfs} 
\usepackage{bbm}

\newtheorem{example}{Example}
\newtheorem{theorem}{Theorem}
\newtheorem{corollary}{Corollary}[theorem]
\newtheorem{lemma}[theorem]{Lemma}
\newtheorem{remark}{Remark}

\title{A Novel Bayes' Theorem for Upper Probabilities}

\author[1]{\href{mailto:<caprio@seas.upenn.edu>?Subject=Your E-pi UAI 2023 paper}{Michele Caprio}{}}
\author[2,3]{Yusuf Sale}
\author[2,3]{Eyke Hüllermeier}
\author[1]{Insup Lee}
\affil[1]{%
    PRECISE Center\\
    Department of Computer and Information Science\\
    University of Pennsylvania\\
    USA
  }

\affil[2]{%
Institute of Informatics\\
University of Munich (LMU)\\
Germany
}

\affil[3]{%
    Munich Center for Machine Learning\\
    Germany

}

\begin{document}
\maketitle

\begin{abstract}
    In their seminal 1990 paper, Wasserman and Kadane establish an upper bound for the Bayes' posterior probability of a measurable set $A$,  when the prior lies in a class of probability measures $\mathcal{P}$ and the likelihood is precise. They also give a sufficient condition for such upper bound to hold with equality. In this paper, we introduce a generalization of their result by additionally addressing uncertainty related to the likelihood. We give an upper bound for the posterior probability when both the prior and the likelihood belong to a set of probabilities. Furthermore, we give a sufficient condition for this upper bound to become an equality. This result is interesting on its own, and has the potential of being applied to various fields of engineering (e.g. model predictive control), machine learning, and artificial intelligence.

    
\end{abstract}

\section{Introduction}\label{intro}

Bayes' rule (BR) is arguably the best-known mechanism to update subjective beliefs. It prescribes the agent to elicit a prior distribution that encapsulates their initial opinion, and to come up with a likelihood that describes the data generating process. Combining prior and likelihood via BR produces the posterior distribution, which captures the agent's revised opinion in light of the collected data. 

But what happens if the agent is not able to specify a single prior distribution? This may occur if they face \textit{ambiguity} \citep{ellsberg,gilboa2}. In \cite[Section 1.1.4]{walley} and \cite{prob.kin}, the authors point out that missing information and bounded rationality may prevent the agent from assessing probabilities precisely in practice, even if doing so is possible in principle. This may be due to the lack of information on how likely events of interest are,  lack of computational time or ability, or because it is extremely difficult to analyze a complex body of evidence. Similarly, the agent may face difficulties in gauging the data generating process, so specifying a single likelihood may become a challenging task. 

The notion of ambiguity is strictly related with that of epistemic uncertainty in machine learning (ML) and artificial intelligence (AI). Let us illustrate this clearly by borrowing concepts from \cite[Section 3.2]{ibnn}. Epistemic uncertainty (EU) corresponds to reducible uncertainty caused by a lack of knowledge about the best model. Notice how, in the precise case -- that is, when the agent specifies a single distribution -- EU is absent. In many applications, a single probability measure is only able to capture the idea of irreducible uncertainty, since it represents a case in which the agent knows exactly the true data generating process, and the prior probability that perfectly encapsulates their initial beliefs. This is a well-studied property of sets of probabilities \citep[Page 458]{eyke}. Due to the increasing relevance of reliable and trustworthy ML and AI applications, effective uncertainty representation and quantification have become vital research areas \citep{kendall2017uncertainties,smith2018understanding, depeweg2018decomposition, kapoor2022uncertainty,sale2023volume,wimmer2023quantifying}. Thus, theoretic underpinnings of imprecise probability theories emerge as a valuable methodology for improving the representation and quantification of uncertainties. Adopting concepts like (convex) sets of probabilities and upper and lower probabilities foster a more sophisticated and fine-tuned articulation of uncertainty.

We remark that EU should not be confused with the concept of  \textit{epistemic probability} \citep{definetti1,definetti2,walley}. In the subjective probability literature, epistemic probability can be captured by a single distribution. Its best definition can be found in \cite[Sections 1.3.2 and 2.11.2]{walley}. There, the author specifies how epistemic probabilities model logical or psychological degrees of partial belief of the agent. We remark, though, how de Finetti and Walley work with finitely additive probabilities, while in this paper we use countably additive probabilities.

The field of imprecise probabilities \citep{augustin,walley}, and in particular the classic paper by \cite{wasserman}, and successive works like \cite{cozman,larry,giaco,kali}, study the problem of an agent updating their beliefs using BR in the presence of ambiguity. 
Our paper belongs to this body of work. Our main result, Theorem \ref{main1}, generalizes the theorem in \cite[Section 2]{wasserman} to the case where the agent faces ambiguity on both what prior and what likelihood to choose to model the phenomenon of interest. We find the upper posterior, that is, the ``upper bound'' to the set of posteriors, using only the upper prior and the upper likelihood. Thanks to the conjugacy property of upper probabilities, introduced in the next section, we derive the lower posterior, that is the ``lower bound'' to the set of posteriors. We also give a necessary condition for the bound to hold with equality. In addition, we hint at possible applications, in particular in the field of model predictive control, a method of process control that is used to control a process while satisfying a set of constraints \cite{hb_mpc}.

The paper is divided as follows. Section \ref{prelim} introduces the concepts that are needed to understand our result.
In Section \ref{result_fixed} we present the main theorem, and conclude our work Section \ref{concl}. We prove our results in Section \ref{proof}.
%
%
\section{Bayes' Theorem for Upper Probabilities}\label{theory_portion}
\subsection{Preliminaries}\label{prelim}
In this section, we introduce the background notions from the IP literature \citep{augustin,decooman,walley} that are needed to understand our main results.

Call $\Delta(\Omega,\mathcal{F})$ the space of (countably additive) probability measures on a generic measurable space $(\Omega,\mathcal{F})$. Pick a generic set $\mathcal{P} \subset \Delta(\Omega,\mathcal{F})$. We denote by $\overline{P}$ the \textit{upper probability}  associated with $\mathcal{P}$, that is, $\overline{P}(A)=\sup_{P\in\mathcal{P}}P(A)$, for all $A\in\mathcal{F}$. Its conjugate is called \textit{lower probability}, $\underline{P}(A)=1-\overline{P}(A^c)=\inf_{P\in\mathcal{P}}P(A)$, for all $A\in\mathcal{F}$. Because of the conjugacy property, in the remainder of this document we focus on upper probabilities only. 

We say that upper probability $\overline{P}$ is \textit{concave} or \textit{$2$-alternating} if $\overline{P}(A\cup B) \leq \overline{P}(A)+\overline{P}(B)-\overline{P}(A\cap B)$, for all $A,B\in\mathcal{F}$. Upper probability $\overline{P}$ is \textit{compatible} \citep{gong} with the set 
\begin{align}
   \text{core}(\overline{P}):=\{&P \in \Delta(\Omega,\mathcal{F}): P(A) \leq \overline{P}(A), \forall A \in \mathcal{F}\} \nonumber\\ 
   =\{&P \in \Delta(\Omega,\mathcal{F}): \overline{P}(A) \geq P(A) \geq \underline{P}(A), \nonumber \\  &\forall A \in \mathcal{F}\} \label{core1}
\end{align}
where \eqref{core1} is a characterization \citep[Page 3389]{cerreia}. The core is the set of all (countably additive) probability measures on $\Omega$ that are set-wise dominated by $\overline{P}$. It is convex: it is immediate to see that if $P$ and $Q$ are dominated by $\overline{P}$, then $\gamma P$ and $(1-\gamma)Q$ are dominated by $\gamma\overline{P}$ and $(1-\gamma)\overline{P}$, respectively, for all $\gamma\in [0,1]$. In turn, $\gamma P+(1-\gamma)Q$ is dominated by $\overline{P}$, thus giving the desired convex property. In addition, throughout the present work, we assume that $\text{core}(\overline{P})$ is nonempty and weak$^\star$-closed.\footnote{Recall that in the weak$^\star$ topology, a net $(P_\alpha)_{\alpha \in I}$ converges to $P$ if and only if $P_\alpha(A) \rightarrow P(A)$, for all $A \in \mathcal{F}$. See also results presented in \cite[Appendix D3]{walley}} Then, as a result of \cite[Section 3.6.1]{walley}, it is also weak$^\star$-compact. 


\begin{remark}
    In \cite[Section 3.6.1]{walley}, the author shows that the finitely additive core is weak$^\star$-compact. The latter is defined as the set of all \textit{finitely additive} probabilities that are set-wise dominated by $\overline{P}$. It is a superset of the countably additive core in \eqref{core1}. To see this, notice that, in general, there might well be a probability measure that is set-wise dominated by $\overline{P}$, but that is merely finitely additive. In fact, there might even be no countably additive probabilities that are set-wise dominated by $\overline{P}$. For this reason, we have to assume that the countably additive core in \eqref{core1} is nonempty. If we further require that the countably additive core in \eqref{core1} is weak$^\star$-closed, then, this implies that it is a weak$^\star$-closed subset of a weak$^\star$-compact space. By \cite[Theorem 2.35]{rudin2}, closed subsets of compact sets are compact. In turn, we have that if the countably additive core is weak$^\star$-closed, it is also weak$^\star$-compact. 
\end{remark}

We now present a class of probabilities that (i) is well-studied and used in robust statistics \cite{huber}, and (ii) is the core of a concave upper probability. 
Other classes with similar properties are presented in \cite[Examples 3-7]{wasserman}. 

\begin{example}[$\varepsilon$-contaminated class]\label{ex_cont}
Consider the space $\Delta(\Omega,\mathcal{F})$ of probability measures on a generic measurable space $(\Omega,\mathcal{F})$, and assume $\Omega$ is compact. Pick any $P\in\Delta(\Omega,\mathcal{F})$ and any $\varepsilon\in [0,1]$. Define
\begin{align}\label{eps_cont_ex}
    \mathcal{Q}^\text{co}:=\{Q\in\Delta(\Omega,\mathcal{F}) : Q(A)=(1-\varepsilon)P(A)+ \varepsilon R(A), \nonumber \\ \forall A\in\mathcal{F} \text{, } R \in \Delta(\Omega,\mathcal{F})\}. 
\end{align}
Superscript ``co'' stands for convex and core. $\mathcal{Q}^\text{co}$ is the $\varepsilon$-contaminated class induced by $P$; it was studied in \cite[Example 3]{wasserman} and references therein. We have that $\overline{Q}(A)=(1-\varepsilon)P(A)+\varepsilon$, for all $A\in\mathcal{F}\setminus\{\emptyset\}$ and $\underline{Q}(A)=(1-\varepsilon)P(A)$, for all $A\in\mathcal{F}\setminus\{\Omega\}$. In addition, $\mathcal{Q}^\text{co}=\text{core}(\overline{Q})$, and $\overline{Q}$ is concave. 
\end{example}

The $\varepsilon$-contaminated class is also instrumental for a future application of Theorem \ref{main1} to model predictive control. We will discuss this at length at the end of section \ref{result_fixed}. There, we also explain why it is important to account for the ambiguity in the likelihood model in real-world safety-critical scenarios.

\subsection{A novel Bayes' Theorem for Upper Probabilities}\label{result_fixed}

Let $(\Theta,\mathcal{F})$ be the measurable parameter space of interest and $\Delta(\Theta,\mathcal{F})$ the space of (countably additive) probability measures on $(\Theta,\mathcal{F})$. 
Let $\mathcal{Y}$ be the set of all \textit{bounded}, \textit{non-negative}, $\mathcal{F}$-measurable functionals on $\Theta$. Call $\mathscr{D}$ the sample space endowed with sigma-algebra $\mathcal{A}$. That is, for any random variable $Y$ of interest and all $\theta\in\Theta$, $Y(\theta)=y\in\mathscr{D}$. Let the agent elicit a set of probabilities $\mathcal{L}_\theta:=\{P_\theta\in\Delta (\mathscr{D},\mathcal{A}) : \theta \in\Theta\}$ on $\mathscr{D}$, parameterized by $\theta\in\Theta$. This captures the ambiguity faced by the agent in determining the true data generating process \citep{ellsberg,gilboa2}. We write $P_\theta \equiv P(\cdot \mid \theta)$ for notational convenience. Assume that each $P_\theta\in\mathcal{L}_\theta$ has density $L(\theta)=p(y\mid \theta)$ with respect to some sigma-finite dominating measure $\nu$ on $(\mathscr{D},\mathcal{A})$; this represents the likelihood function for $\theta$ having observed data $y\in \mathscr{D}$.

\textbf{Assumption 1.} Every density $L$ corresponding to an element $P_\theta$ of $\mathcal{L}_\theta$ belongs to $\mathcal{Y}$; that is, every density is bounded and non-negative.

Assumption 1 is needed mainly for mathematical purposes; as we shall see later in this section, it can be relaxed. Let the agent specify a set $\mathcal{P}$ of probabilities on $(\Theta,\mathcal{F})$. It represents their (incomplete) prior knowledge on the elements of $\mathcal{F}$; its elements may be informed by the collected data, thus giving the analysis an (imprecise) empirical Bayes flavor \citep{casella}. Then, compute $\overline{P}(A)=\sup_{P\in\mathcal{P}}P(A)$, for all $A\in\mathcal{F}$, and consider $\mathcal{P}^{\text{co}}:=\text{core}(\overline{P})$, assumed nonempty and weak$^\star$-closed. It represents the agent's initial beliefs. We assume that every $P\in\mathcal{P}^{\text{co}}$ has density $p$ with respect to some sigma-finite dominating measure $\mu$ on $(\Theta,\mathcal{F})$, that is, $p={\text{d}P}/{\text{d}\mu}$. We require the agent's beliefs to be represented by the core for two main reasons. The first, mathematical, one is to ensure that the upper probability is compatible with the belief set. The second, philosophical, one is the following. In Bayesian statistics, the agent selects a specific prior to encapsulate their initial beliefs. \cite{berger} points out how such choice is oftentimes arbitrary, and posits the \textit{dogma of ideal precision} (DIP). It states that in any problem there is an \textit{ideal probability model} $P_T$ which is precise, but which may not be precisely known. To overcome this shortcoming, the agent should specify a finite collection $\{P_s\}_{s\in\mathcal{S}}$ of ``plausible'' prior distributions, and compute the posterior for each $P_s$. Notice how this corresponds to selecting a finite number of elements from the core of $\overline{P}_\mathcal{S}$, where $\overline{P}_\mathcal{S}(A)=\sup_{s\in\mathcal{S}}P_s(A)$, for all $A\in\mathcal{F}$. A criticism to the DIP was brought forward by Walley. In \cite[Section 2.10.4.(c)]{walley}, he claims how given an upper probability $\overline{P}$, there is no cogent reason for which the agent should choose a specific $P_T$ that is dominated by $\overline{P}$, or -- for that matter -- a collection $\{P_s\}_{s\in\mathcal{S}}$ of ``plausible'' probabilities. Because the core considers all (countably additive) probability measures that are dominated by $\overline{P}$, it is the perfect instrument to address Walley's criticism \citep{prob.kin}.

Let the agent compute $\overline{P}_\theta$, the upper probability associated with $\mathcal{L}_\theta$, and consider $\mathcal{L}^{\text{co}}_\theta:=\text{core}(\overline{P}_\theta)$, assumed nonempty and weak$^\star$-closed. It represents the set of plausible likelihoods. As \cite{grun} point out, accounting for ambiguity around the true data generating process is crucial, as Bayesian inference may suffer from inconsistency issues if carried out using a misspecified likelihood.

Let 
\begin{equation}\label{L_scr}
    \mathscr{L}:=\left\lbrace{L=\frac{\text{d}P_\theta}{\text{d}\nu} \text{, } P_\theta \in \mathcal{L}_\theta^{\text{co}}}\right\rbrace \subset \mathcal{Y}
\end{equation}
be the set of pdf/pmf associated with the elements of $\mathcal{L}^{\text{co}}_\theta$. Let also  $\overline{L}(\theta):=\sup_{L\in\mathscr{L}}L(\theta)$ and $\underline{L}(\theta):=\inf_{L\in\mathscr{L}}L(\theta)$, for all $\theta\in\Theta$. Call 
\begin{align*}
    \mathcal{P}^{\text{co}}_y:=\Bigg\{ P_y &\in \Delta(\Theta,\mathcal{F}) : \\
    \frac{\text{d}P_y}{\text{d}\mu}&=p(\theta\mid y)=\frac{L(\theta)p(\theta)}{\int_\Theta L(\theta)p(\theta) \text{d}\theta}\text{, }\\  
    p&=\frac{\text{d}P}{\text{d}\mu} \text{, } P\in\mathcal{P}^{\text{co}} \text{, } L=\frac{\text{d}P_\theta}{\text{d}\nu} \text{, } P_\theta\in\mathcal{L}_\theta^{\text{co}}\Bigg\}
\end{align*}
the class of posterior probabilities when the prior is in $\mathcal{P}^{\text{co}}$ and the likelihood is in $\mathcal{L}_\theta^{\text{co}}$, and let $\overline{P}_y(A)=\sup_{P_y\in\mathcal{P}^{\text{co}}_y}P_y(A)$, for all $A\in\mathcal{F}$. Then, the following is a generalization of Bayes' theorem in \cite[Section 2]{wasserman}, and is an extension of \cite[Theorem 7]{ibnn}. We prove it in Section \ref{proof}.

\begin{theorem}[Bayes' theorem for upper probabilities]\label{main1}
Suppose $\mathcal{P}^{\text{co}},\mathcal{L}_\theta^{\text{co}}$ are nonempty and weak$^\star$-closed. Then for all $A\in\mathcal{F}$,
\begin{align}
   \overline{P}_y(A) &\leq \frac{\sup_{P\in\mathcal{P}^{\text{co}}}\int_\Theta \overline{L}(\theta) \mathbbm{1}_A(\theta) P(\text{d}\theta)}{\mathbf{c}} \label{ineq_main} \\ 
   &\leq \frac{\int_0^\infty \overline{P}\left( \left\lbrace{\theta\in\Theta : \overline{L}(\theta)\mathbbm{1}_A(\theta)>t}\right\rbrace \right) \text{d}t}{\mathbf{c}^\prime}, \label{ineq_main2} \end{align}
provided that the ratios are well defined. Here $\mathbbm{1}_A$ denotes the indicator function for $A\in\mathcal{F}$, $\mathbf{c}:=\sup_{P\in\mathcal{P}^{\text{co}}}\int_\Theta \overline{L}(\theta) \mathbbm{1}_A(\theta) P(\text{d}\theta)+ \inf_{P\in\mathcal{P}^{\text{co}}}\int_\Theta \underline{L}(\theta) \mathbbm{1}_{A^c}(\theta) P(\text{d}\theta)$, and 
\begin{align*}
    \mathbf{c}^\prime :=&\underbrace{\int_0^\infty \overline{P}\left( \left\lbrace{\theta\in\Theta : \overline{L}(\theta)\mathbbm{1}_A(\theta)>t}\right\rbrace \right) \text{d}t}_{\text{\textit{upper Choquet integral} of } \overline{L}\mathbbm{1}_A}\\ + &\underbrace{\int_0^\infty \underline{P}\left( \left\lbrace{\theta\in\Theta : \underline{L}(\theta)\mathbbm{1}_{A^c}(\theta)>t}\right\rbrace \right) \text{d}t}_{\text{\textit{lower Choquet integral} of } \underline{L}\mathbbm{1}_{A^c}}.
\end{align*}
In addition, if $\overline{P}$ is concave, then the inequalities in \eqref{ineq_main} and \eqref{ineq_main2} are equalities, for all $A\in\mathcal{F}$.
\end{theorem}
This result is particularly appealing. Under Assumption 1, if the prior upper probability (PUP) is concave and the prior and likelihood sets $\mathcal{P}^{\text{co}},\mathcal{L}_\theta^{\text{co}}$ are nonempty and weak$^\star$-closed, then the agent can perform a (generalized) Bayesian update of the PUP by carrying out only one operation. This is the case even when the agent faces ambiguity around the true data generating process so that a set of likelihoods is needed. The posterior lower probability is obtained immediately via the conjugacy property $\underline{P}_y(A)=1-\overline{P}_y(A^c)$.

\begin{corollary}\label{cor_main1}
    Retain the assumptions in Theorem  \ref{main1}. If $\mathcal{L}^\text{co}_\theta$ is a singleton, we retrieve Bayes' theorem in \cite[Section 2]{wasserman}.
\end{corollary}

Corollary \ref{cor_main1} tell us that when there is no ambiguity around the likelihood, Theorem \ref{main1} recovers Wasserman and Kadane's classical Bayes' theorem. Given the straightforward nature of the proof, we omit it here. We also have the following lemma, that is proved in Section \ref{proof}.



\begin{lemma}[Preserved concavity]\label{lemma_ccv}
Suppose $\mathcal{P}^{\text{co}},\mathcal{L}_\theta^{\text{co}}$ are nonempty and weak$^\star$-closed. Then, if $\overline{P}$ is concave, we have that $\overline{P}_y$ is concave as well.
\end{lemma}
This lemma is important because it tells us that the generalized Bayesian update of Theorem \ref{main1} preserves concavity, and so it can be applied to successive iterations. If at time $t$ the PUP is concave, then the PUP at time $t+1$ -- that is, the posterior upper probability at time $t$ -- will be concave too. Necessary and sufficient conditions for a generic upper probability to be concave are given in \cite[Section 5]{marinacci2}.

In the future, we plan to forego Assumption 1 and use the techniques developed in \cite{decooman} to generalize our results to the case in which the elements of $\mathscr{L}$ are unbounded and not necessarily non-negative. We also intend to extend our results to the case in which the elements of $\mathcal{Y}$ are $\mathbb{R}^{d}$-valued, for some $d\in\mathbb{N}$. We suspect this is a less demanding endeavor since we do not use specific properties of $\mathbb{R}$ in our proofs.

As mentioned earlier, a natural application of our results is model predictive control (MPC). MPC is a method that is used to control a process while satisfying a set of constraints \citep{hb_mpc}. Typically, when the process is stable (or at least stable for the past $k$ time steps, for some $k\geq 0$) and if the scholar decides to take the Bayesian approach, they proceed as follows. They specify a Normal likelihood (the distribution of the control inputs) centered at the objective function of the process, and a Normal prior on the parameter of such function. In this framework, if ambiguity enters the picture, then our results become relevant. 

The importance of addressing prior ambiguity was discussed at length in section \ref{result_fixed}. The reasons why accounting for likelihood ambiguity is important are as follows. First, as pointed out earlier, we may run into inconsistency issues if we perform Bayesian analysis using a misspecified likelihood \citep{grun}. Second, a (precise) Normal likelihood is a good choice only if stability of the process is ensured. MPC procedures are used in the process industries in chemical plants, oil refineries, power system balancing models, and power electronics. These are all safety-critical applications where accounting for possible sudden unexpected instabilities is of paramount importance.

The scholar may specify a class of $\epsilon$-contaminated truncated Normal priors and a class of $\eta$-contaminated truncated Normal likelihoods, and use Theorem \ref{main1} to compute the upper posterior. Notice that the Normals need to be truncated in light of Assumption 1. This requirement is not too stringent, and -- as pointed out earlier in this section -- our future work will allow us do without it.



\section{Conclusion}\label{concl}
In this paper, we present a new Bayes' theorem for upper probabilities that extends the one in \cite[Section 2]{wasserman}, and \cite[Theorem 7]{ibnn}. In the future, we plan to generalize Theorem \ref{main1} by letting go of Assumption 1, and to apply it to an MPC problem and to other fields of engineering, and ML and AI. For example, we intend to use it to overcome the computational bottleneck of step 2 of the algorithm that computes the posterior set in an imprecise Bayesian neural network procedure \citep{ibnn}. There, an element-wise application of Bayes' rule for all the extreme elements of the prior and likelihood sets is performed. As we can see, this is a combinatorial task that can potentially be greatly simplified in light of Theorem \ref{main1}, conveying a computationally cheaper algorithm.

\section{Proofs}\label{proof}
\begin{proof}[Proof of Theorem \ref{main1}]
Assume that $\mathcal{P}^{\text{co}},\mathcal{L}_\theta^{\text{co}}$ are nonempty and weak$^\star$-closed. Pick any $A\in\mathcal{F}$. Recall that we can rewrite the usual Bayes' updating rule as
\begin{align*}
    {P}_y(A) &= \frac{\int_\Theta {L}(\theta) \mathbbm{1}_A(\theta) P(\text{d}\theta)}{\int_\Theta {L}(\theta) \mathbbm{1}_A(\theta) P(\text{d}\theta)+ \int_\Theta {L}(\theta) \mathbbm{1}_{A^c}(\theta) P(\text{d}\theta)}\\
    &=\frac{1}{1+\frac{\int_\Theta {L}(\theta) \mathbbm{1}_{A^c}(\theta) P(\text{d}\theta)}{\int_\Theta {L}(\theta) \mathbbm{1}_{A}(\theta) P(\text{d}\theta)}},
\end{align*}
which is maximized when 
$$\frac{\int_\Theta {L}(\theta) \mathbbm{1}_{A^c}(\theta) P(\text{d}\theta)}{\int_\Theta {L}(\theta) \mathbbm{1}_{A}(\theta) P(\text{d}\theta)}$$
is minimized. But
$$\frac{\int_\Theta {L}(\theta) \mathbbm{1}_{A^c}(\theta) P(\text{d}\theta)}{\int_\Theta {L}(\theta) \mathbbm{1}_{A}(\theta) P(\text{d}\theta)} \geq \frac{\inf_{P\in\mathcal{P}^{\text{co}}}\int_\Theta \underline{L}(\theta) \mathbbm{1}_{A^c}(\theta) P(\text{d}\theta)}{\sup_{P\in\mathcal{P}^{\text{co}}}\int_\Theta \overline{L}(\theta) \mathbbm{1}_A(\theta) P(\text{d}\theta)},$$
which proves the inequality in \eqref{ineq_main}. The inequality in \eqref{ineq_main2} is true because
\begin{align*}
   \inf_{P\in\mathcal{P}^{\text{co}}}&\int_\Theta \underline{L}(\theta) \mathbbm{1}_{A^c}(\theta) P(\text{d}\theta) \\
   \geq &\int_0^\infty \underline{P}\left( \left\lbrace{\theta\in\Theta : \underline{L}(\theta)\mathbbm{1}_{A^c}(\theta)>t}\right\rbrace \right) \text{d}t 
\end{align*}
and 
\begin{align*}
    \sup_{P\in\mathcal{P}^{\text{co}}}&\int_\Theta \overline{L}(\theta) \mathbbm{1}_A(\theta) P(\text{d}\theta)\\
    \leq &\int_0^\infty \overline{P}\left( \left\lbrace{\theta\in\Theta : \overline{L}(\theta)\mathbbm{1}_{A}(\theta)>t}\right\rbrace \right) \text{d}t.
\end{align*}
Assume now that $\overline{P}$ is concave. By \cite[Lemma 1]{wasserman}, we have that there exists $\check{P}\in\mathcal{P}^{\text{co}}$ such that 
\begin{align}\label{eq_lemma1}
    \sup_{P\in\mathcal{P}^{\text{co}}}\int_\Theta L(\theta)\mathbbm{1}_A(\theta)P(\text{d}\theta)=\int_\Theta L(\theta)\mathbbm{1}_A(\theta)\check{P}(\text{d}\theta),
\end{align}
for all $L\in\mathscr{L}$. In addition, by \cite[Lemma 4]{wasserman}, we have that for all $Y\in\mathcal{Y}$ and all $\epsilon>0$, there exists a non-negative, upper semi-continuous function $h\leq Y$ such that
 \begin{align}\label{eq_pre_lemma}
 \begin{split}
\left[\sup_{P\in\mathcal{P}^{\text{co}}}\int_\Theta Y(\theta)P(\text{d}\theta)\right] - \epsilon &< \sup_{P\in\mathcal{P}^{\text{co}}}\int_\Theta h(\theta)P(\text{d}\theta) \\&\leq \sup_{P\in\mathcal{P}^{\text{co}}}\int_\Theta Y(\theta)P(\text{d}\theta).
\end{split}
\end{align}
Let now $Y=\overline{L}\mathbbm{1}_A$. Notice that since $\mathcal{L}^\text{co}_\theta$ is weak$^\star$-compact (as a result of \cite[Section 3.6.1]{walley}), by \eqref{L_scr} so is $\mathscr{L}$. This implies that $\underline{L},\overline{L}\in\mathscr{L}$, since a compact set always contains its boundary, so $\underline{L},\overline{L}\in\mathcal{Y}$ as well, and in turn $\underline{L}\mathbbm{1}_{A^c},\overline{L}\mathbbm{1}_A\in\mathcal{Y}$. Fix then any $L\in\mathscr{L}$ and put $h=L\mathbbm{1}_A$. It is immediate to see that $h$ is non-negative and upper semi-continuous. Then, by \eqref{eq_pre_lemma}, we have that for all $\epsilon>0$
\begin{align}\label{eq_lemma4}
\begin{split}
&\left[\sup_{P\in\mathcal{P}^{\text{co}}}\int_\Theta \overline{L}(\theta)\mathbbm{1}_A(\theta)P(\text{d}\theta)\right] - \epsilon < \\&\sup_{P\in\mathcal{P}^{\text{co}}}\int_\Theta {L}(\theta)\mathbbm{1}_A(\theta)P(\text{d}\theta) \leq \sup_{P\in\mathcal{P}^{\text{co}}}\int_\Theta \overline{L}(\theta)\mathbbm{1}_A(\theta)P(\text{d}\theta).
\end{split}
\end{align}
Combining \eqref{eq_lemma1} and\eqref{eq_lemma4}, we obtain
\begin{align}\label{eq_lemma_comb}
\begin{split}
    &\left[\sup_{P\in\mathcal{P}^{\text{co}}}\int_\Theta \overline{L}(\theta)\mathbbm{1}_A(\theta)P(\text{d}\theta)\right] - \epsilon\\ &< \int_\Theta {L}(\theta)\mathbbm{1}_A(\theta)\check{P}(\text{d}\theta) \leq \sup_{P\in\mathcal{P}^{\text{co}}}\int_\Theta \overline{L}(\theta)\mathbbm{1}_A(\theta)P(\text{d}\theta),
\end{split}
\end{align}
for all $L\in\mathscr{L}$. 

Pick now any $\epsilon>0$ and put 
\begin{align*}
    k&:=\sup_{P\in\mathcal{P}^{\text{co}}}\int_\Theta \overline{L}(\theta) \mathbbm{1}_A(\theta) P(\text{d}\theta)\\&+ \inf_{P\in\mathcal{P}^{\text{co}}}\int_\Theta \underline{L}(\theta) \mathbbm{1}_{A^c}(\theta) P(\text{d}\theta)>0.
\end{align*}
Choose any $L\in\mathscr{L}$ and $\delta\in (0,\epsilon k)$. By \eqref{eq_lemma_comb} we have that 
\begin{equation}\label{ineq_imm_1}
    \left[\sup_{P\in\mathcal{P}^{\text{co}}}\int_\Theta \overline{L}(\theta)\mathbbm{1}_A(\theta)P(\text{d}\theta)\right] - \delta < \int_\Theta {L}(\theta)\mathbbm{1}_A(\theta)\check{P}(\text{d}\theta)
\end{equation}
and that 
\begin{equation}\label{ineq_imm_2}
    \left[\inf_{P\in\mathcal{P}^{\text{co}}}\int_\Theta \underline{L}(\theta)\mathbbm{1}_{A^c}(\theta)P(\text{d}\theta)\right] + \delta > \int_\Theta {L}(\theta)\mathbbm{1}_{A^c}(\theta)\check{P}(\text{d}\theta).
\end{equation} 
The inequality in \eqref{ineq_imm_1} comes from the fact that the first inequality in \eqref{eq_lemma_comb} holds for all $\epsilon>0$, and -- given how $k$ is defined -- we have that $\delta>0$. The inequality in \eqref{ineq_imm_2} is obtained by re-deriving \eqref{eq_lemma1}, \eqref{eq_pre_lemma}, \eqref{eq_lemma4}, and \eqref{eq_lemma_comb} for the infimum of set $\mathcal{P}^\text{co}$ rather than the supremum. In that case, we simply substitute $\sup$ with $\inf$, $\overline{L}$ with $\underline{L}$, $\mathbbm{1}_A$ with $\mathbbm{1}_{A^c}$, ``$-\epsilon$'' with ``$+\epsilon$'', and we reverse the inequalities. 

Recall that $\mathbf{c}:=\sup_{P\in\mathcal{P}^{\text{co}}}\int_\Theta \overline{L}(\theta) \mathbbm{1}_A(\theta) P(\text{d}\theta)+ \inf_{P\in\mathcal{P}^{\text{co}}}\int_\Theta \underline{L}(\theta) \mathbbm{1}_{A^c}(\theta) P(\text{d}\theta)$, and define $$\mathbf{d}:=\int_\Theta {L}(\theta) \mathbbm{1}_A(\theta) \check{P}(\text{d}\theta)+ \int_\Theta {L}(\theta) \mathbbm{1}_{A^c}(\theta) \check{P}(\text{d}\theta).$$ Then we have, 
\begin{align*}
    \check{P}_y(A)&=\frac{\int_\Theta {L}(\theta) \mathbbm{1}_A(\theta) \check{P}(\text{d}\theta)}{\mathbf{d}}\\
    &\geq \frac{\left[\sup_{P\in\mathcal{P}^{\text{co}}}\int_\Theta \overline{L}(\theta) \mathbbm{1}_A(\theta) P(\text{d}\theta)\right]-\delta}{\mathbf{c}+\delta-\delta}\\
    &=\frac{\sup_{P\in\mathcal{P}^{\text{co}}}\int_\Theta \overline{L}(\theta) \mathbbm{1}_A(\theta) P(\text{d}\theta)}{\mathbf{c}} -\frac{\delta}{k}\\
    &>\frac{\sup_{P\in\mathcal{P}^{\text{co}}}\int_\Theta \overline{L}(\theta) \mathbbm{1}_A(\theta) P(\text{d}\theta)}{\mathbf{c}} -\epsilon.
\end{align*}
Since this holds for all $\epsilon>0$, we have that
$$\sup_{P_y\in\mathcal{P}_y^{\text{co}}}P_y(A)=\frac{\sup_{P\in\mathcal{P}^{\text{co}}}\int_\Theta \overline{L}(\theta) \mathbbm{1}_A(\theta) P(\text{d}\theta)}{\mathbf{c}},$$
concluding the proof of inequality \eqref{ineq_main} being an equality when $\overline{P}$ is concave. Inequality \eqref{ineq_main2} being an equality when $\overline{P}$ is concave comes immediately from \cite[Lemma 4]{wasserman}, and the fact that $\overline{L}\mathbbm{1}_A,\underline{L}\mathbbm{1}_{A^c}\in\mathcal{Y}$, as pointed out above.
\end{proof}

\begin{proof}[Proof of Lemma \ref{lemma_ccv}]
In their works \cite{walley2,wasserman}, the authors show that concave upper probabilities are closed with respect to the generalized Bayes' rule. In particular, this means that, if we let $\mathbf{b}:=\sup_{P\in\mathcal{P}^{\text{co}}}\int_\Theta {L}(\theta) \mathbbm{1}_A(\theta) P(\text{d}\theta)+ \inf_{P\in\mathcal{P}^{\text{co}}}\int_\Theta {L}(\theta) \mathbbm{1}_{A^c}(\theta) P(\text{d}\theta)$, for any fixed $A\in\mathcal{F}$, if $\overline{P}$ is concave, then for all $L\in\mathscr{L}$ 
\begin{equation}\label{ccv_gen}
    \overline{P}_y(A)=\frac{\sup_{P\in\mathcal{P}^{\text{co}}}\int_\Theta {L}(\theta) \mathbbm{1}_A(\theta) P(\text{d}\theta)}{\mathbf{b}}
\end{equation}
is concave. But since $\mathcal{L}^\text{co}_\theta$ is weak$^\star$-compact (as a consequence of \cite[Section 3.6.1]{walley}), by \eqref{L_scr} so is $\mathscr{L}$. This implies that $\underline{L},\overline{L}\in\mathscr{L}$, since a compact set always contains its boundary. Call then $L^\prime=\overline{L}\mathbbm{1}_A+\underline{L}\mathbbm{1}_{A^c}$. It is immediate to see that $L^\prime\in\mathscr{L}$. Then, by \eqref{ccv_gen} we have that if we call $\mathbf{b}^\prime:=\sup_{P\in\mathcal{P}^{\text{co}}}\int_\Theta {L^\prime}(\theta) \mathbbm{1}_A(\theta) P(\text{d}\theta)+ \inf_{P\in\mathcal{P}^{\text{co}}}\int_\Theta {L^\prime}(\theta) \mathbbm{1}_{A^c}(\theta) P(\text{d}\theta)$, it follows that
\begin{align*}
    \overline{P}_y(A)&=\frac{\sup_{P\in\mathcal{P}^{\text{co}}}\int_\Theta {L^\prime}(\theta) \mathbbm{1}_A(\theta) P(\text{d}\theta)}{\mathbf{b}^\prime}\\
    &=\frac{\sup_{P\in\mathcal{P}^{\text{co}}}\int_\Theta \overline{L}(\theta) \mathbbm{1}_A(\theta) P(\text{d}\theta)}{\mathbf{c}}
\end{align*}
is concave, concluding the proof.
\end{proof}

\begin{contributions}
Michele Caprio and Yusuf Sale contributed equally to this paper.

\end{contributions}

\begin{acknowledgements} 
Michele Caprio would like to acknowledge partial funding by the Army Research Office (ARO MURI W911NF2010080). Yusuf Sale is supported by the DAAD programme Konrad Zuse Schools of Excellence in Artificial Intelligence, sponsored by the Federal Ministry of Education and Research.
\end{acknowledgements}

\bibliography{uai2023-template}
\end{document}